\relax
\documentclass[letterpaper]{article}
\usepackage{ijcai17}
\usepackage{times}
\usepackage{helvet}
\usepackage{courier}
\usepackage{amssymb}
\usepackage[boxruled,linesnumbered,vlined]{algorithm2e}
\usepackage{amsmath}
\usepackage[font=scriptsize,labelfont=bf]{caption}
\usepackage{graphicx}
\usepackage{graphics}
\usepackage{subcaption}
\usepackage{float}
\usepackage{amsthm}
\usepackage{mathtools}

\newtheorem{theorem}{Theorem}
\newtheorem{lemma}{Lemma}

\makeatletter
\newcommand*{\centerfloat}{%
  \parindent \z@
  \leftskip \z@ \@plus 1fil \@minus \textwidth
  \rightskip\leftskip
  \parfillskip \z@skip}
\makeatother

\begin{document}
\title{The FastMap Algorithm for Shortest Path Computations}
\author{Liron Cohen$^1$ \ Tansel Uras$^1$ \ Shiva Jahangiri$^2$ \ Aliyah Arunasalam$^1$ \ Sven Koenig$^1$ \ T. K. Satish Kumar$^1$
\\ \{lironcoh, turas, arunasal, skoenig\}@usc.edu, shivaj@uci.edu, tkskwork@gmail.com
\\ $^1$University of Southern California \ \ \ $^2$University of California, Irvine}
\maketitle

\begin{abstract}
We present a new preprocessing algorithm for embedding the nodes of a given edge-weighted undirected graph into a Euclidean space. The Euclidean distance between any two nodes in this space approximates the length of the shortest path between them in the given graph. Later, at runtime, a shortest path between any two nodes can be computed with A* search using the Euclidean distances as heuristic. Our preprocessing algorithm, called FastMap, is inspired by the data mining algorithm of the same name and runs in near-linear time. Hence, FastMap is orders of magnitude faster than competing approaches that produce a Euclidean embedding using Semidefinite Programming. FastMap also produces admissible and consistent heuristics and therefore guarantees the generation of shortest paths. Moreover, FastMap applies to general undirected graphs for which many traditional heuristics, such as the Manhattan Distance heuristic, are not well defined. Empirically, we demonstrate that A* search using the FastMap heuristic is competitive with A* search using other state-of-the-art heuristics, such as the Differential heuristic.
\end{abstract}

\section{Introduction and Related Work}

Shortest path computations commonly occur in the inner procedures of many AI programs. In video games, for example, a large fraction of CPU cycles is spent on shortest path computations \cite{UK:AIPRO:15}. Many other tasks in AI, including motion planning \cite{L:BOOK:06}, temporal reasoning \cite{D:BOOK:03}, and decision making \cite{RN:BOOK:09}, also involve finding and reasoning about shortest paths. While Dijkstra's algorithm \cite{DIJ:NM:59} can be used to compute shortest paths in polynomial time, speeding up shortest path computations allows one to solve the aformentioned tasks faster. One way to do that is to use A* search with an informed heuristic \cite{HNR:IEEE:68}.

A perfect heuristic is one that returns the true distance between any two nodes in a given graph. A* with such a heuristic and proper tie-breaking is guaranteed to expand nodes only on a shortest path between the given start and goal nodes. In general, computing the perfect heuristic between two nodes is as hard as computing the shortest path between them. Hence, A* search benefits from a perfect heuristic only if it is computed offline. However, precomputing all pairwise distances is not only time-intensive but also requires a prohibitive $O(N^2)$ memory where $N$ is the number of nodes. The memory requirements for storing all-pairs shortest paths data can be somewhat addressed through compression \cite{BH:ICAPS:13,SBH:JAIR:15}.

Existing methods for preprocessing a given graph (without precomputing all pairwise distances) can be grouped into the following categories: Hierarchical abstractions that yield suboptimal paths have been used to reduce the size of the search space by abstracting groups of nodes \cite{BMS:JGD:04,SB:AAAI:05}. More informed heuristics \cite{BH:AIIDE:06,C:IEEE:06,SFBS:AAAI:09} focus A* searches better, resulting in fewer expanded states. Hierarchies can also be used to derive heuristics during the search \cite{LRH:SOCS:08,HDPZM:CAI:94}. Dead-end detection and other pruning methods \cite{BH:AIIDE:06,GFSS:SOCS:10,PZR:AAAI:10} identify areas of the graph that do not need to be searched to find shortest paths. Search with contraction hierarchies \cite{GSSD:CEA:08} is an optimal hierarchical method, where every level of the hierarchy contains only a single node. It has been shown to be efficient on road networks but seems to be less efficient on graphs with higher branching factors, such as grid-based game maps \cite{S:AAI:13}. N-level graphs \cite{UK:AAAI:15}, constructed from undirected graphs by partitioning the nodes into levels also allow for significant pruning during the search.

A different approach that does not rely on preprocessing of the graph is to use some notion of a geometric distance between two nodes as a heuristic of the distance between them. One such heuristic for gridworlds is the Manhattan Distance heuristic.\footnote{In a 4-neighbor 2D gridworld, for example, the Manhattan Distance between two cells $(x_1,y_1)$ and $(x_2,y_2)$ is $|x_1-x_2|+|y_1-y_2|$. Generalizations exist for 8-neighbor 2D and 3D gridworlds.} For many gridworlds, A* search using the Manhattan Distance heuristic outperforms Dijkstra's algorithm. However, in complicated 2D/3D gridworlds like mazes, the Manhattan Distance heuristic may not be sufficiently informed to focus A* searches effectively. Another issue associated with Manhattan Distance-like heuristics is that they are not well defined for general graphs.\footnote{Henceforth, whenever we refer to a graph, we mean an edge-weighted undirected graph unless stated otherwise.} For a graph that cannot be conceived in a geometric space, there is no closed-form formula for a ``geometric'' heuristic for the distance between two nodes because there are no coordinates associated with them.

For a graph that does not already have a geometric embedding in Euclidean space, a preprocessing algorithm can be used to generate one. As described before, at runtime, A* search would then use the Euclidean distance between any two nodes in this space as an estimate for the distance between them in the given graph. One such approach is Euclidean Heuristic Optimization (EHO) \cite{RBS:AAAI:11}. EHO guarantees admissiblility and consistency and therefore generates shortest paths. However, it requires solving a Semidefinite Program (SDP). SDPs can be solved in polynomial time \cite{VB:SR:96}. EHO leverages additional structure to solve them in cubic time. Still, a cubic preprocessing time limits the size of the graphs that are amenable to this approach.

The Differential heuristic is another state-of-the-art approach that has the benefit of a near-linear runtime. However, unlike the approach in \cite{RBS:AAAI:11}, it does not produce an explicit Euclidean embedding. In the preprocessing phase of the Differential heuristic approach, some nodes of the graph are chosen as pivot nodes. The distances between each pivot node and every other node are precomputed and stored \cite{SFBS:AAAI:09}. At runtime, the heuristic between two nodes $a$ and $b$ is given by $\max_p |d(a,p)-d(p,b)|$, where $p$ is a pivot node and $d(\cdot,\cdot)$ is the precomputed distance. The preprocessing time is linear in the number of pivots times the size of the graph. The required space is linear in the number of pivots times the number of nodes, although a more succinct representation is presented in \cite{GSFS:AAAI:11}. Similar preprocessing techniques are used in Portal-Based True Distance heuristics \cite{GFSS:SOCS:10}.

In this paper, we present a new preprocessing algorithm, called FastMap, that produces an explicit Euclidean embedding while running in near-linear time. It therefore has the benefits of the small preprocessing time of the Differential heuristic approach and of producing an embedding from which a heuristic between two nodes can be quickly computed using a closed-form formula. Our preprocessing algorithm, dubbed FastMap, is inspired by the data mining algorithm of the same name \cite{FL:SIGMOD:95}. It is orders of magnitude faster than SDP-based approaches for producing Euclidean embeddings. FastMap also produces admissible and consistent heuristics and therefore guarantees the generation of shortest paths.

The FastMap heuristic has several advantages: First, it is defined for general (undirected) graphs. Second, we observe empirically that, in gridworlds, A* using the FastMap heuristic runs faster than A* using the Manhattan or Octile distance heuristics. A* using the FastMap heuristic runs equally fast or faster than A* using the Differential heuristic, with the same memory resources. The (explicit) Euclidean embedding produced by FastMap also has representational benefits like recovering the underlying manifolds of the graph and/or visualizing them. Moreover, we observe that the FastMap and Differential heuristics have complementary strengths and can be easily combined to generate an even more informed heuristic.

\section{The Origin of FastMap}
The FastMap algorithm \cite{FL:SIGMOD:95} was introduced in the data mining community for automatically generating geometric embeddings of abstract objects. For example, if we are given objects in form of long DNA strings, multimedia datasets such as voice excerpts or images, or medical datasets such as ECGs or MRIs, there is no geometric space in which these objects can be naturally visualized. However, there is often a well defined distance function between every pair of objects. For example, the \emph{edit distance}\footnote{The edit distance between two strings is the minimum number of insertions, deletions or substitutions that are needed to transform one to the other.} between two DNA strings is well defined although an individual DNA string cannot be conceptualized in geometric space.

Clustering techniques, such as the $k$-means algorithm, are well studied in machine learning \cite{A:BOOK:10} but cannot be applied directly to domains with abstract objects because they assume that objects are described as points in geometric space. FastMap revives their applicability by first creating a Euclidean embedding for the abstract objects that approximately preserves the pairwise distances between them. Such an embedding also helps to visualize the abstract objects, for example, to aid physicians in identifying correlations between symptoms from medical records.

The data mining FastMap gets as input a complete undirected edge-weighted graph $G=(V,E)$. Each node $v_i \in V$ represents an abstract object $O_i$. Between any two nodes $v_i$ and $v_j$ there is an edge $(v_i,v_j) \in E$ with weight $D(O_i,O_j)$ that corresponds to the given distance between objects $O_i$ and $O_j$. A Euclidean embedding assigns to each object $O_i$ a $K$-dimensional point $p_i \in \mathbb{R}^K$. A good Euclidean embedding is one in which the Euclidean distance between any two points $p_i$ and $p_j$ closely approximates $D(O_i,O_j)$.

One of the early approaches for generating such an embedding is based on the idea of \emph{multi-dimensional scaling} (MDS) \cite{T:PSY:52}. Here, the overall distortion of the pairwise distances is measured in terms of the ``energy'' stored in ``springs'' that connect each pair of objects. MDS, however, requires $O(|V|^2)$ time and hence does not scale well in practice. On the other hand, FastMap \cite{FL:SIGMOD:95} requires only linear time. Both methods embed the objects in a $K$-dimensional space for a user-specified $K$.

FastMap works as follows: In the very first iteration, it heuristically identifies the farthest pair of objects $O_a$ and $O_b$ in linear time. It does this by initially choosing a random object $O_b$ and then choosing $O_a$ to be the farthest object away from $O_b$. It then reassigns $O_b$ to be the farthest object away from $O_a$. Once $O_a$ and $O_b$ are determined, every other object $O_i$ defines a triangle with sides of lengths $d_{ai}=D(O_a,O_i)$, $d_{ab}=D(O_a,O_b)$ and $d_{ib}=D(O_i,O_b)$. Figure \ref{fig1} shows this triangle. The sides of the triangle define its entire geometry, and the projection of $O_i$ onto $O_a O_b$ is given by $x_i = (d_{ai}^2 + d_{ab}^2 - d_{ib}^2) / (2d_{ab})$. FastMap sets the first coordinate of $p_i$, the embedding of object $O_i$, to $x_i$. In particular, the first coordinate of $p_a$ is $x_a=0$ and of $p_b$ is $x_b=d_{ab}$. Computing the first coordinates of all objects takes only linear time since the distance between any two objects $O_i$ and $O_j$ for $i,j \notin \{a,b\}$ is never computed.

\begin{figure}[t]
%\centerfloat
  \centering
  \includegraphics[width=0.3\textwidth]{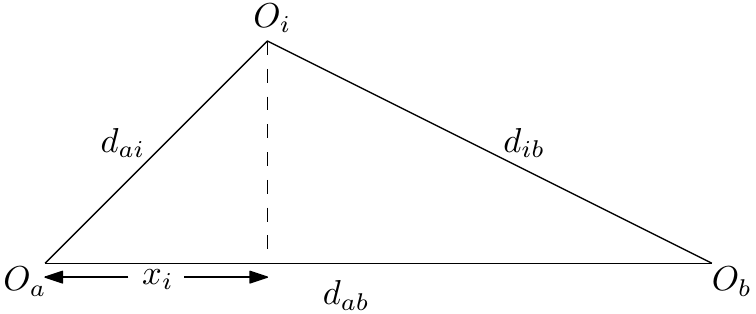}
  \caption{The three sides of a triangle define its entire geometry. In particular, $x_i = (d_{ai}^2 + d_{ab}^2 - d_{ib}^2) / (2d_{ab})$.}
    \label{fig1}
\end{figure}

In the subsequent $K-1$ iterations, the same procedure is followed for computing the remaining $K-1$ coordinates of each object. However, the distance function is adapted for different iterations. For example, for the first iteration, the coordinates of $O_a$ and $O_b$ are $0$ and $d_{ab}$, respectively. Because these coordinates fully explain the true distance between them, from the second iteration onwards, the rest of $p_a$ and $p_b$'s coordinates should be identical. Intuitively, this means that the second iteration should mimic the first one on a hyperplane that is perpendicular to the line $O_a O_b$. Figure \ref{fig2} explains this intuition. Although the hyperplane is never constructed explicitly, its conceptualization implies that the distance function for the second iteration should be changed in the following way: $D_{new}(O'_i,O'_j)^2 = D (O_i, O_j)^2 - (x_i - x_j)^2$. Here, $O'_i$ and $O'_j$ are the projections of $O_i$ and $O_j$, respectively, onto this hyperplane, and $D_{new}$ is the new distance function.

\begin{figure}[t]
%\centerfloat
  \centering
  \includegraphics[width=0.33\textwidth]{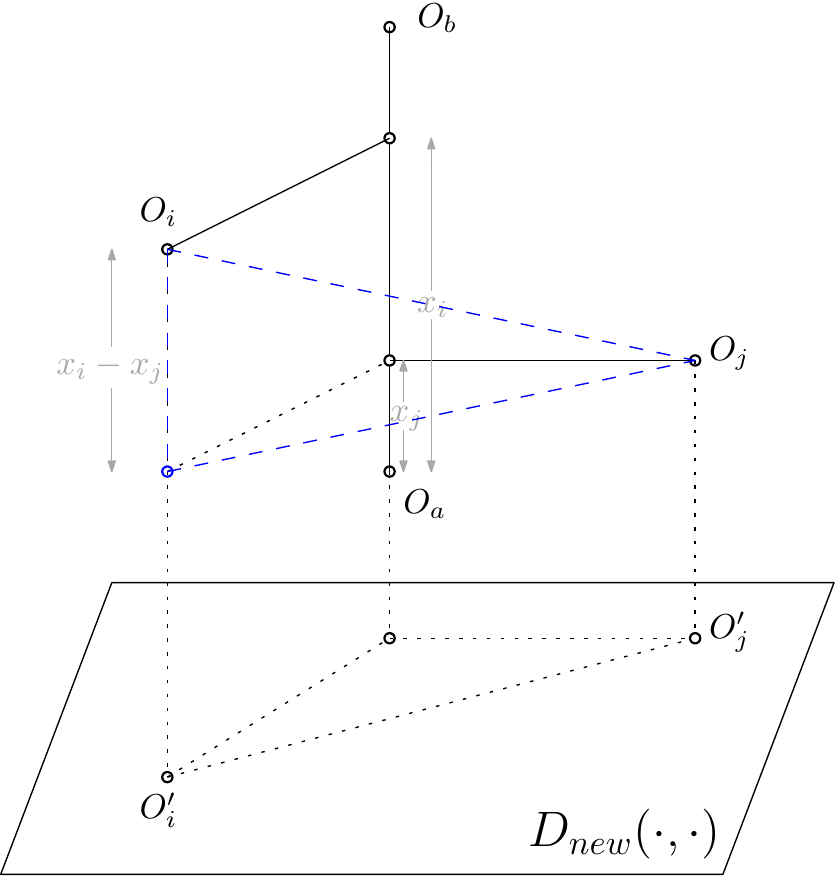}
  \caption{Shows a geometric conceptualization of the recursive step in FastMap. In particular, $D_{new}(O'_i,O'_j)^2 = D(O_i, O_j)^2 - (x_i - x_j)^2$.}
    \label{fig2}
\end{figure}

\section{FastMap for Shortest Path Computations}
In this section, we provide the high-level ideas for how to adapt the data mining FastMap algorithm to shortest path computations. In the shortest path computation problem, we are given a non-negative edge-weighted undirected graph $G=(V,E)$ along with a start node $v_s$ and a goal node $v_g$. As a preprocessing technique, we can embed the nodes of $G$ in a Euclidean space. As A* searches for a shortest path from $v_s$ to $v_g$, it can use the Euclidean distance from $v \in V$ to $v_g$ as a heuristic for $v$. The number of node expansions of A* search depends on the informedness of the heuristic which, in turn, depends on the ability of the embedding to preserve the pairwise distances.

The idea is to view the nodes of $G$ as the objects to be embedded in Euclidean space. As such, the data mining FastMap algorithm cannot directly be used for generating an embedding in linear time. The data mining FastMap algorithm assumes that the distance $d_{ij}$ between two objects $O_i$ and $O_j$ can be computed in constant time, independent of the number of objects. Computing the distance between two nodes depends on the size of the graph. Another problem is that the Euclidean distances may not satisfy important properties such as \emph{admissibility} or \emph{consistency}. Admissibility guarantees that A* finds shortest paths while consistency allows A* to avoid re-expansions of nodes as well.

The first issue of having to retain (near-)linear time complexity can be addressed as follows: In each iteration, after we identify the farthest pair of nodes $O_a$ and $O_b$, the distances $d_{ai}$ and $d_{ib}$ need to be computed for \emph{all} other nodes $O_i$. Computing $d_{ai}$ and $d_{ib}$ for any single node $O_i$ can no longer be done in constant time but requires $O(|E|+|V| \log |V|)$ time instead \cite{TF:SFCS:84}. However, since we need to compute these distances for all nodes, computing two shortest path trees rooted at nodes $O_a$ and $O_b$ yields all necessary distances. The complexity of doing so is also $O(|E|+|V| \log |V|)$, which is only linear in the size of the graph.\footnote{unless $|E|=O(|V|)$, in which case the complexity is near-linear in the size of the input because of the $\log|V|$ factor} The amortized complexity for computing $d_{ai}$ and $d_{ib}$ for any single node $O_i$ is therefore near-constant time.

The second issue of having to generate a consistent (and thus admissible) heuristic is formally addressed in Theorem \ref{thm_consistency}. The idea is to use $L_1$ distances instead of $L_2$ distances in each iteration of FastMap. The mathematical properties of the $L_1$ distance can be used to prove that admissibility and consistency hold irrespective of the dimensionality of the embedding.

\RestyleAlgo{algoruled}
    \begin{algorithm}
%      \SetAlgoRefName{} %no number
%      \NoCaptionOfAlgo  %no algorithm:
      \KwIn{$G = (V,E,w)$, $K_{max}$ and $\epsilon$.}
      \KwOut{$K$ and $p_i \in \mathbb{R}^K$ for all $v_i \in V$.}
      \BlankLine
	  $w' = w$; $\ \ K = 1$\;
	  \While{$K_{max} > 0$}{
	    Let $G' = (V, E, w')$\;
	    $(n_a, n_b) \leftarrow $ GetFarthestPair($G'$)\;
	    Compute shortest path trees rooted at $n_a$ and $n_b$ on $G'$ to obtain $d_{ab}$, $d_{ai}$ and $d_{ib}$ for all $v_i \in V$\;
%	  	Run Dijkstra from $n_a$ and $n_b$ on $G' = (V, E, w')$ to compute $d_{ab}$, $d_{ai}$ and $d_{ib}$ for $i=\{1, \ldots, N\}$\;
	  	\If{$d_{ab} < \epsilon$}{
	  	  Break\;
	  	}
	  	\For{each $v \in V$}{
	  	  $[p_v]_K = (d_{av} + d_{ab} - d_{vb})/2$% \tcp*[r]{K$^{th}$ coord.}
	  	}
	  	\For{each edge $(u,v) \in E$}{
			 $w'(u,v) = w'(u,v) - | [p_u]_K - [p_v]_K |$\;
	  	}
	  	$K=K+1$; $\ \ K_{max}=K_{max}-1$\;
	  }
  \caption{Shows the FastMap algorithm. $G = (V,E,w)$ is a non-negative edge-weighted undirected graph; $K_{max}$ is the user-specified upper bound on the dimensionality; $\epsilon$ is a user-specified threshold; $K \leq K_{max}$ is the dimensionality of the computed embedding; $p_i$ is the Euclidean embedding of node $v_i \in V$. Line $11$ is equivalent to $w'(u,v) = w(u,v) - \lVert p_u - p_v \rVert_1$.}
  \label{alg1}
\end{algorithm}

Algorithm \ref{alg1} presents data mining FastMap adapted to the shortest path computation problem. The input is an edge-weighted undirected graph $G = (V,E,w)$ along with two user-specified parameters $K_{max}$ and $\epsilon$. $K_{max}$ is the maximum number of dimensions allowed in the Euclidean embedding. It bounds the amount of memory needed to store the Euclidean embedding of any node. $\epsilon$ is the threshold that marks a point of diminishing returns when the distance between the farthest pair of nodes becomes negligible. The output is an embedding $p_i \in \mathbb{R}^K$ (with $K \leq K_{max}$) for each node $v_i \in V$.

The algorithm maintains a working graph $G'=(V, E, w')$ initialized to $G$. The nodes and edges of $G'$ are always identical to those of $G$ but the weights on the edges of $G'$ change with every iteration. In each iteration, the farthest pair ($n_a,n_b$) of nodes in $G'$ is heuristically identified in near-linear time (line $4$). The $K^{th}$ coordinate $[p_i]_K$ of each node $v_i$ is computed using a formula similar to that for $x_i$ in Figure \ref{fig1}. However, that formula is modified to $(d_{ai} + d_{ab} - d_{ib})/2$ to ensure admissibility and consistency of the heuristic. In each iteration, the weight of each edge is decremented to resemble the update rule for $D_{new}$ in Figure \ref{fig2} (line $11$). However, that update rule is modified to $w'(u,v) = w'(u,v) - | [p_u]_K - [p_v]_K |$ to use the $L_1$ distances instead of the $L_2$ distances.%Theorem \ref{thm_consistency} shows that doing so ensures admissibility and consistency of the heuristic.

The method GetFarthestPair($G'$) (line $4$) computes shortest path trees in $G'$ a small constant number of times, denoted by $\tau$.\footnote{$\tau=10$ in our experiments.} It therefore runs in near-linear time. In the first iteration, we assign $n_a$ to be a random node. A shortest path tree rooted at $n_a$ is computed to identify the farthest node from it. $n_b$ is assigned to be this farthest node. In the next iteration, a shortest path tree rooted at $n_b$ is computed to identify the farthest node from it. $n_a$ is reassigned to be this farthest node. Subsequent iterations follow the same switching rule for $n_a$ and $n_b$. The final assignments of nodes to $n_a$ and $n_b$ are returned after $\tau$ iterations. This entire process of starting from a randomly chosen node can be repeated a small constant number of times.\footnote{This constant is also $10$ in our experiments.}

Figure \ref{fig_example} shows the working of our algorithm on a small gridworld example.

\begin{figure*}[!]
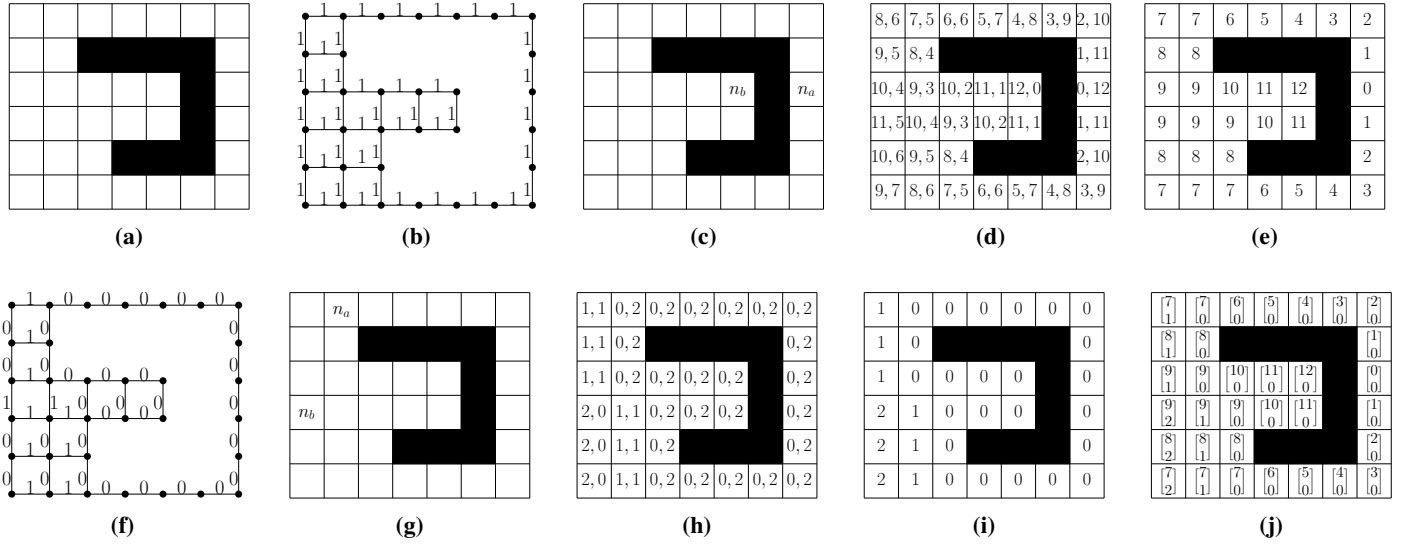

\centerfloat
 
  \begin{subfigure}[b]{0.2\textwidth}
    \centering
	\includegraphics[width=0.9\textwidth]{Figs/fig3_1.pdf}
    \caption{}
  \end{subfigure}
  ~
  \begin{subfigure}[b]{0.2\textwidth}
    \centering
	\includegraphics[width=0.9\textwidth]{Figs/fig3_3.pdf}
    \caption{}
  \end{subfigure}
  ~
  \begin{subfigure}[b]{0.2\textwidth}
    \centering
	\includegraphics[width=0.9\textwidth]{Figs/fig3_4.pdf}
    \caption{}
  \end{subfigure}
  ~
  \begin{subfigure}[b]{0.2\textwidth}
    \centering
	\includegraphics[width=0.9\textwidth]{Figs/fig3_5.pdf}
    \caption{}
  \end{subfigure}
  \begin{subfigure}[b]{0.2\textwidth}
    \centering
	\includegraphics[width=0.9\textwidth]{Figs/fig3_6.pdf}
    \caption{}
  \end{subfigure}
  
  \vspace{0.5cm}
  
  \begin{subfigure}[b]{0.2\textwidth}
    \centering
	\includegraphics[width=0.9\textwidth]{Figs/fig3_7.pdf}
    \caption{}
  \end{subfigure}
  ~
  \begin{subfigure}[b]{0.2\textwidth}
    \centering
	\includegraphics[width=0.9\textwidth]{Figs/fig3_8.pdf}
    \caption{}
  \end{subfigure}
  ~
  \begin{subfigure}[b]{0.2\textwidth}
    \centering
	\includegraphics[width=0.9\textwidth]{Figs/fig3_9.pdf}
    \caption{}
  \end{subfigure}
  ~
  \begin{subfigure}[b]{0.2\textwidth}
    \centering
	\includegraphics[width=0.9\textwidth]{Figs/fig3_10.pdf}
    \caption{}
  \end{subfigure}
  ~
  \begin{subfigure}[b]{0.2\textwidth}
    \centering
	\includegraphics[width=0.9\textwidth]{Figs/fig3_2.pdf}
    \caption{}
  \end{subfigure}
  
  \caption{Illustrates the working of FastMap. (a) shows a $4$-neighbor gridworld with obstacles in black. (b) shows the graphical representation of (a) with the original unit weights on the edges. (c) shows the identified farthest pair of nodes. (d) shows two numbers in each cell representing the distances from $n_a$ and $n_b$, respectively. (e) shows the first coordinate produced for each cell. (f) shows new edge weights for the next iteration. (g), (h) and (i) correspond to (c), (d) and (e), respectively, in the second iteration. (j) shows the produced $2$D embedding.}
  \label{fig_example}
\end{figure*}

\subsection{Proof of Consistency}
In this section, we prove the consistency of the FastMap heuristic. Since consistency implies admissibility, this also proves that A* with the FastMap heuristic returns shortest paths. We use the following notation in the proofs: $w^i_{xy}$ is the weight on the edge between nodes $x$ and $y$ in the $i^{th}$ iteration; $d^i_{xy}$ is the distance between nodes $x$ and $y$ in the $i^{th}$ iteration (using the weights $w^i$); $p_x$ is the vector of coordinates produced for node $x$, and $[p_x]_j$ is its $j^{th}$ coordinate;\footnote{The $i^{th}$ iteration sets the value of $[p_x]_i$.} $h^i_{xy}$ is the FastMap heuristic between nodes $x$ and $y$ after $i$ iterations. Note that $h^i_{xy}$ is the $L_1$ distance between $p_x$ and $p_y$ at iteration $i$, that is $h^i_{xy} \vcentcolon= \sum_{j=1}^{i} |[p_x]_j - [p_y]_j|$. We also define $\Delta^{i+1}_{xy} \vcentcolon= d^i_{xy} - d^{i+1}_{xy}$. In the following proofs, we use the fact that $|A|+|B| \geq |A+B|$ and $|A|-|B| \leq |A-B|$.

\begin{lemma}
  For all $x$, $y$ and $i$, $d^i_{xy} \geq 0$.
\end{lemma}
\begin{proof}
  We prove by induction that in any iteration $i$, $w^i_{uv} \geq 0$ for all $(u,v) \in E$. Thus, the weight of each edge in the $i^{th}$ iteration is non-negative and therefore $d^i_{uv} \geq 0$ for all $u$, $v$. For the base case, $w^1_{uv} = w(u,v) \geq 0$. We assume that $w^{i}_{uv} \geq 0$ and show that $w^{i+1}_{uv} \geq 0$. Let $n_a$ and $n_b$ be the farthest pair of nodes identified in the $i^{th}$ iteration. From lines $9$ and $11$, $w^{i+1}_{uv} = w^i_{uv} - | (d^i_{au} - d^i_{av}) / 2 + (d^i_{vb} - d^i_{ub}) / 2|$. To show that $w^{i+1}_{uv} \geq 0$ we show that $w^i_{uv} \geq | (d^i_{au} - d^i_{av}) / 2 + (d^i_{vb} - d^i_{ub}) / 2|$. From the triangle inequality, for any node $l$, $d^i_{uv} + \min(d^i_{ul},d^i_{lv}) \geq \max(d^i_{ul},d^i_{lv})$. Therefore, $d^i_{uv} \geq |d^i_{lv} - d^i_{ul}|$. This means that $d^i_{uv} \geq |d^i_{au} - d^i_{av}|/2 + |d^i_{vb} - d^i_{ub}|/2$. Therefore, $d^i_{uv} \geq | (d^i_{au} - d^i_{av}) / 2 + (d^i_{vb} - d^i_{ub}) / 2|$. This concludes the proof since $w^i_{uv} \geq d^i_{uv}$.
\end{proof}

\begin{lemma}
  For all $x$, $y$ and $i$, $\Delta^{i+1}_{xy} \geq | [p_x]_i - [p_y]_i |$.
\end{lemma}
\begin{proof}
  Let $\langle u_1=x,\ldots,u_m=y \rangle$ be the shortest path from $x$ to $y$ in iteration $i$. By definition, $d^i_{xy} = \sum_{j=1}^{m-1} w^i_{u_j u_{j+1}}$ and $d^{i+1}_{xy} \leq \sum_{j=1}^{m-1} w^{i+1}_{u_j u_{j+1}}$. From line $11$, $w^{i+1}_{u_j u_{j+1}} = w^i_{u_j u_{j+1}} - | [p_{u_j}]_i - [p_{u_{j+1}}]_i |$. Therefore, $\Delta^{i+1}_{xy} = d^i_{xy} - d^{i+1}_{xy} \geq \sum_{j=1}^{m-1} | [p_{u_j}]_i - [p_{u_{j+1}}]_i |$. This concludes the proof since $\sum_{j=1}^{m-1} | [p_{u_j}]_i - [p_{u_{j+1}}]_i | \geq | \sum_{j=1}^{m-1} [p_{u_j}]_i - [p_{u_{j+1}}]_i | = | [p_x]_i - [p_y]_i |$.
\end{proof}

\begin{lemma}
  For all $x$, $y$, $g$ and $i$, $d^1_{xy} + h^i_{yg} - h^i_{xg} \geq d^{i+1}_{xy}$.
\end{lemma}
\begin{proof}
  We prove the lemma by induction on $i$. The base case for $i=1$ is implied by Lemma $2$. We assume that $d^1_{xy} + h^i_{yg} - h^i_{xg} \geq d^{i+1}_{xy}$ and show $d^1_{xy} + h^{i+1}_{yg} - h^{i+1}_{xg} \geq d^{i+2}_{xy}$. We know that $h^{i+1}_{yg} - h^{i+1}_{xg} = h^i_{yg} - h^i_{xg} - (|[p_x]_{i+1}-[p_g]_{i+1}| - |[p_y]_{i+1}-[p_g]_{i+1}|)$. Since $|[p_x]_{i+1}-[p_g]_{i+1}| - |[p_y]_{i+1}-[p_g]_{i+1}| \leq |[p_x]_{i+1} - [p_y]_{i+1}|$, we have $h^{i+1}_{yg} - h^{i+1}_{xg} \geq h^i_{yg} - h^i_{xg} - |[p_x]_{i+1} - [p_y]_{i+1}|$. Hence, $d^1_{xy} + h^{i+1}_{yg} - h^{i+1}_{xg} \geq (d^1_{xy} + h^i_{yg} - h^i_{xg}) - |[p_x]_{i+1} - [p_y]_{i+1}|$. Using the inductive assumption, we get $d^1_{xy} + h^{i+1}_{yg} - h^{i+1}_{xg} \geq d^{i+1}_{xy} - |[p_x]_{i+1} - [p_y]_{i+1}|$. By definition, $d_{xy}^{i+1} = \Delta_{xy}^{i+2} + d_{xy}^{i+2}$. Substituting for $d_{xy}^{i+1}$, we get $d^1_{xy} + h^{i+1}_{yg} - h^{i+1}_{xg} \geq d_{xy}^{i+2} +(\Delta_{xy}^{i+2} - |[p_x]_{i+1} - [p_y]_{i+1}|)$. Lemma $2$ shows that $\Delta_{xy}^{i+2} \geq |[p_x]_{i+1} - [p_y]_{i+1}|$, which concludes the proof.
\end{proof}

\begin{theorem}
  The FastMap heuristic is consistent.
  \label{thm_consistency}
\end{theorem}
\begin{proof}
  For all $x$, $y$, $g$ and $i$: From Lemma $3$, we know $d^1_{xy} + h^i_{yg} - h^i_{xg} \geq d^{i+1}_{xy}$. From Lemma $1$, we know $d^{i+1}_{xy} \geq 0$. Put together, we have $d^1_{xy} + h^i_{yg} \geq h^i_{xg}$. Finally, $h^i_{gg} = \sum_{j=1}^{i} |[p_g]_j - [p_g]_j| = 0$.
\end{proof}

\begin{theorem}
  The informedness of the FastMap heuristic increases monotonically with the number of dimensions.
  \label{thm_informedness}
\end{theorem}
\begin{proof}
  This theorem follows from the fact that for any two nodes $x$ and $g$, $h^{i+1}_{xg} = h^i_{xg} + |[p_x]_{i+1} - [p_g]_{i+1}| \geq h^i_{xg}$.
\end{proof}

\section{Experimental Results}
We performed experiments on many benchmark maps from \cite{S:HOG:12}. Figure \ref{fig_experiments} presents representative results. The FastMap heuristic (FM) and the Differential heuristic (DH) with equal memory resources\footnote{The dimensionality of the Euclidean embedding for FM matches the number of pivots in DH.} are compared against each other. In addition, we include the Octile heuristic (OCT) as a baseline, that also uses a closed-form formula for the computation of its heuristic.

We observe that, as the number of dimensions increases, (a) FM and DH perform better than OCT; (b) the median number of expanded nodes when using the FM heuristic decreases (which is consistent with Theorem \ref{thm_informedness}); and (c) the median absolute deviation (MAD) of the number of expanded nodes when using the FM heuristic decreases. When FM's MADs are high, the variabilities can possibly be exploited in future work using Rapid Randomized Restart strategies.

FastMap also gives us a framework of identifying a point of diminishing returns with increasing dimensionality. This happens when the distance between the farthest pair of nodes stops being ``significant''. For example, such a point is observed in Figure \ref{fig_experiments}(f) around dimensionality $5$.\footnote{The distances between the farthest pair of nodes, computed on line $4$ of Algorithm \ref{alg1}, for the first $10$ dimensions are: $\langle 581, 36, 22, 15, 14, 10, 6, 6, 5, 4 \rangle$.}

In mazes, such as in Figure \ref{fig_experiments}(g), A* using the DH heuristic outperforms A* using the FM heuristic. This leads us to believe that FM provides good heuristic guidance in domains that can be approximated with a low-dimensional manifold. This observation also motivates us to create a hybrid FM+DH heuristic by taking the maximum of the two heuristics. Some relevant results are shown in Table \ref{table1}. We use FM($K$) to denote the FM heuristic with $K$ dimensions and DM($K$) to denote the DH heuristic with $K$ pivots. For the results in Table \ref{table1}, all heuristics have equal memory resources. We observe that the number of node expansions of A* using the FM($5$)+DH($5$) heuristic is always second best compared to A* using the FM($10$) heuristic and A* using the DH($10$) heuristic. On one hand, this decreases the percentages of instances on which it expands the least number of nodes (as seen in the second row of Table \ref{table1}). But, on the other hand, its median number of node expansions is not far from that of the best technique in each breakdown.

\begin{table*}[]
\centering
\resizebox{\linewidth}{!}{%
\begin{tabular}{|c|r|r|r|r|r|r|r|r|r|r|r|r|r|r|r|r|r|r|}
\hline
Map         & \multicolumn{6}{c|}{`lak503d'}                                                                                                                                  & \multicolumn{6}{c|}{`brc300d'}                                                                                                                                  & \multicolumn{6}{c|}{`maze512-32-0'}                                                                                                                             \\ \hline
            & \multicolumn{2}{c|}{FM-WINS 570}                    & \multicolumn{2}{c|}{DH-WINS 329}                    & \multicolumn{2}{c|}{FM+DH-WINS 101}                 & \multicolumn{2}{c|}{FM-WINS 846}                    & \multicolumn{2}{c|}{DH-WINS 147}                    & \multicolumn{2}{c|}{FM+DH-WINS 7}                   & \multicolumn{2}{c|}{FM-WINS 382}                    & \multicolumn{2}{c|}{DH-WINS 507}                    & \multicolumn{2}{c|}{FM+DH-WINS 111}                 \\ \hline
            & \multicolumn{1}{c|}{Med} & \multicolumn{1}{c|}{MAD} & \multicolumn{1}{c|}{Med} & \multicolumn{1}{c|}{MAD} & \multicolumn{1}{c|}{Med} & \multicolumn{1}{c|}{MAD} & \multicolumn{1}{c|}{Med} & \multicolumn{1}{c|}{MAD} & \multicolumn{1}{c|}{Med} & \multicolumn{1}{c|}{MAD} & \multicolumn{1}{c|}{Med} & \multicolumn{1}{c|}{MAD} & \multicolumn{1}{c|}{Med} & \multicolumn{1}{c|}{MAD} & \multicolumn{1}{c|}{Med} & \multicolumn{1}{c|}{MAD} & \multicolumn{1}{c|}{Med} & \multicolumn{1}{c|}{MAD} \\ \hline
FM(10)      & 261                      & 112                      & 465                      & 319                      & 2,222                     & 1,111                     & 205                      & 105                      & 285                      & 149                      & 894                      & 472                      & 1,649                     & 747                      & 11,440                    & 9,861                     & 33,734                    & 13,748                    \\ \hline
DH(10)      & 358                      & 215                      & 278                      & 156                      & 885                      & 370                      & 217                      & 119                      & 200                      & 129                      & 277                      & 75                       & 3,107                     & 2,569                     & 2,859                     & 2,194                     & 8,156                     & 4,431                     \\ \hline
FM(5)+DH(5) & 303                      & 160                      & 323                      & 170                      & 610                      & 264                      & 206                      & 105                      & 267                      & 135                      & 249                      & 73                       & 2,685                     & 2,091                     & 3,896                     & 2,992                     & 7,439                     & 4,247                     \\ \hline
\end{tabular}
}
\caption{Shows the median and MAD numbers of A* node expansions for different maps using three different heuristics with equal memory resources on $1000$ random instances. FM($10$) denotes the FastMap heuristic with $10$ dimensions, DH($10$) denotes the Differential heuristic with $10$ pivots and FM($5$)+DH($5$) is a combined heuristic which takes the maximum of a $5$-dimensional FastMap heuristic and a $5$-pivot Differential heuristic. The results are split into bins according to winners (along with their number of wins).}
\label{table1}
\end{table*}

\begin{figure*}[!]
\centerfloat
 
  \begin{subfigure}[b]{0.3\textwidth}
    \centering
	\includegraphics[width=0.9\textwidth]{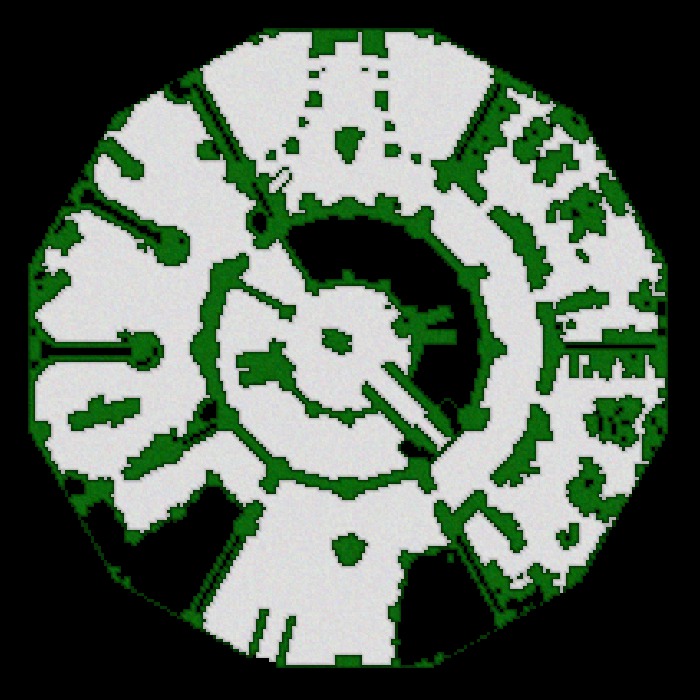}
    \caption{}
  \end{subfigure}
  ~
  \begin{subfigure}[b]{0.35\textwidth}
    \centering
	\includegraphics[width=1\textwidth]{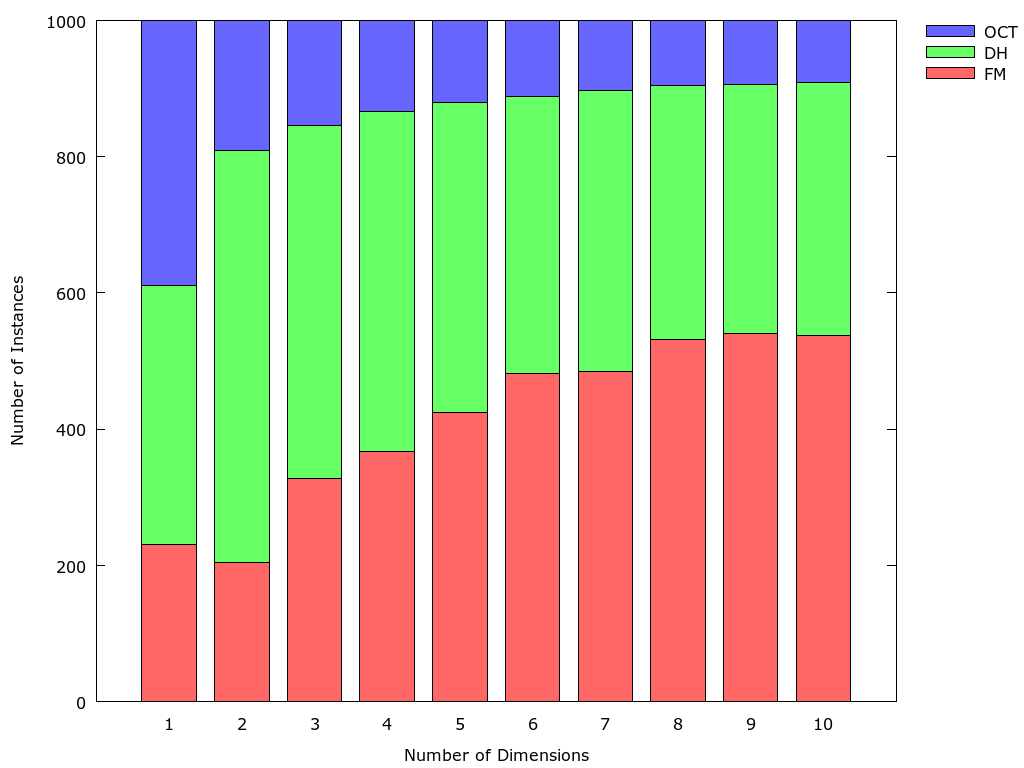}
    \caption{}
  \end{subfigure}
  ~
  \begin{subfigure}[b]{0.35\textwidth}
    \centering
	\includegraphics[width=1\textwidth]{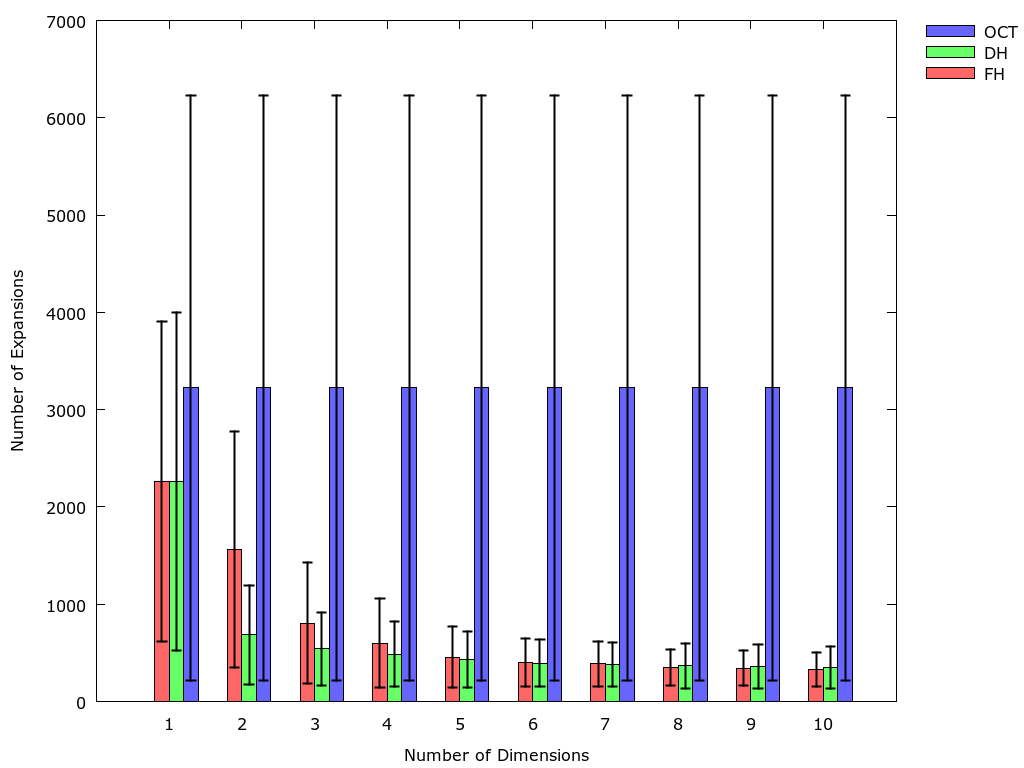}
    \caption{}
  \end{subfigure}
  
    \vspace{0.5cm}

  \begin{subfigure}[b]{0.3\textwidth}
    \centering
	\includegraphics[width=0.9\textwidth]{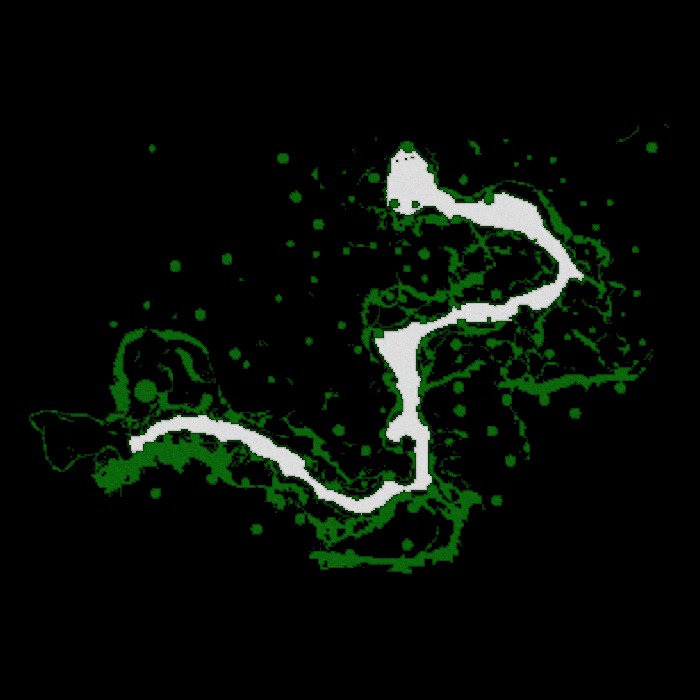}
    \caption{}
  \end{subfigure}
  ~
  \begin{subfigure}[b]{0.35\textwidth}
    \centering
	\includegraphics[width=1\textwidth]{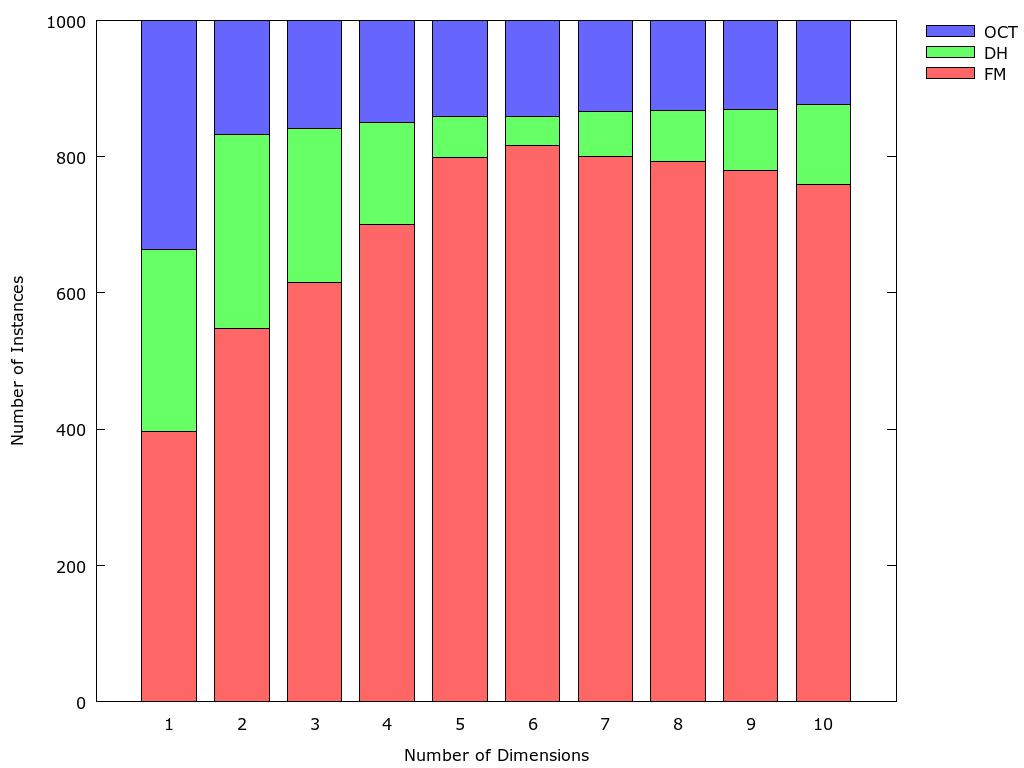}
    \caption{}
  \end{subfigure}
  ~
  \begin{subfigure}[b]{0.35\textwidth}
    \centering
	\includegraphics[width=1\textwidth]{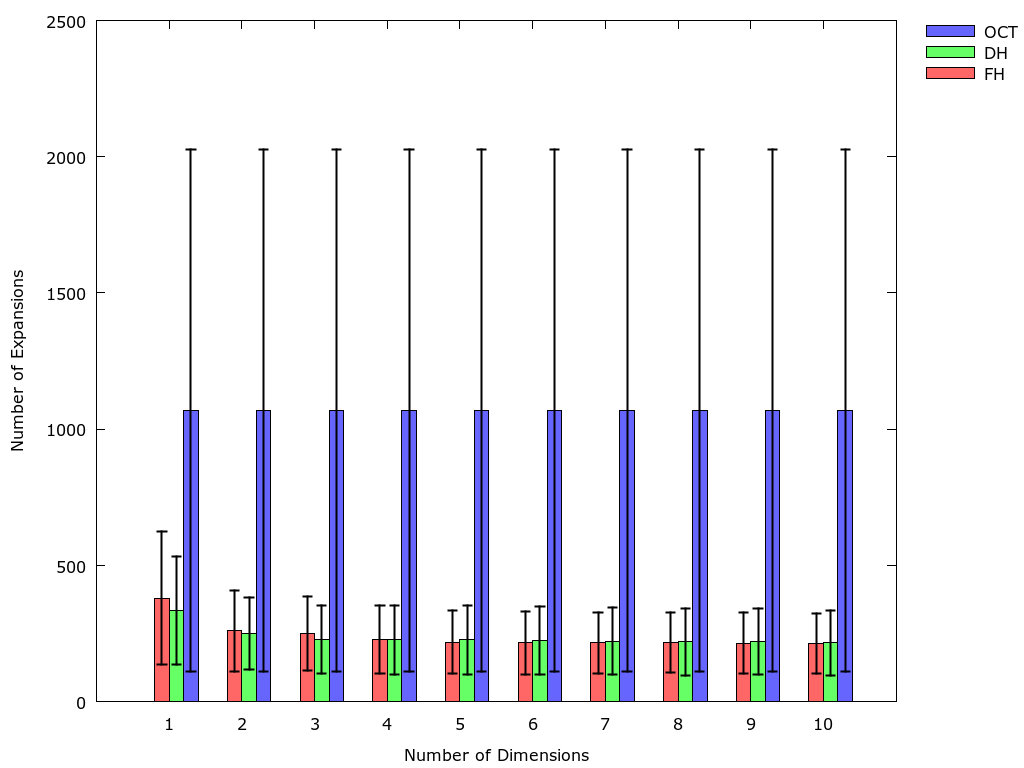}
    \caption{}
  \end{subfigure}
  
  \vspace{0.5cm}
  
  \begin{subfigure}[b]{0.3\textwidth}
    \centering
	\includegraphics[width=0.9\textwidth]{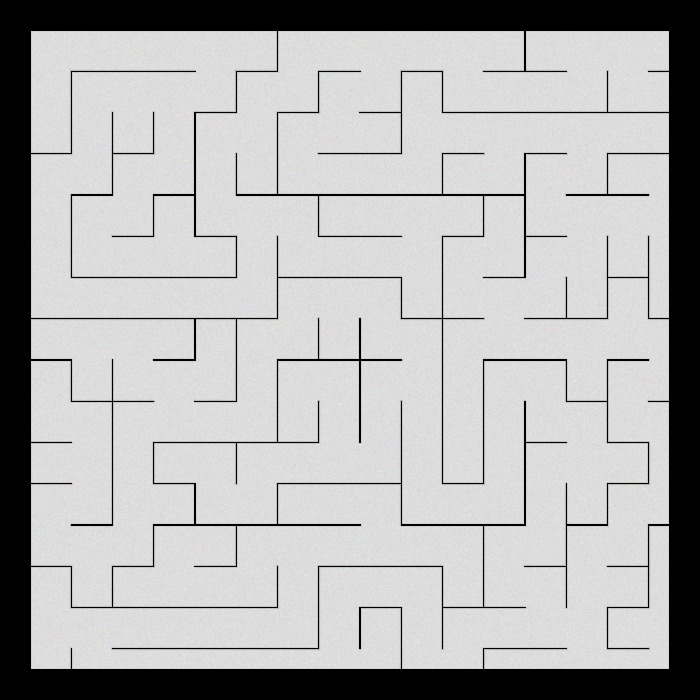}
    \caption{}
  \end{subfigure}
  ~
  \begin{subfigure}[b]{0.35\textwidth}
    \centering
	\includegraphics[width=1\textwidth]{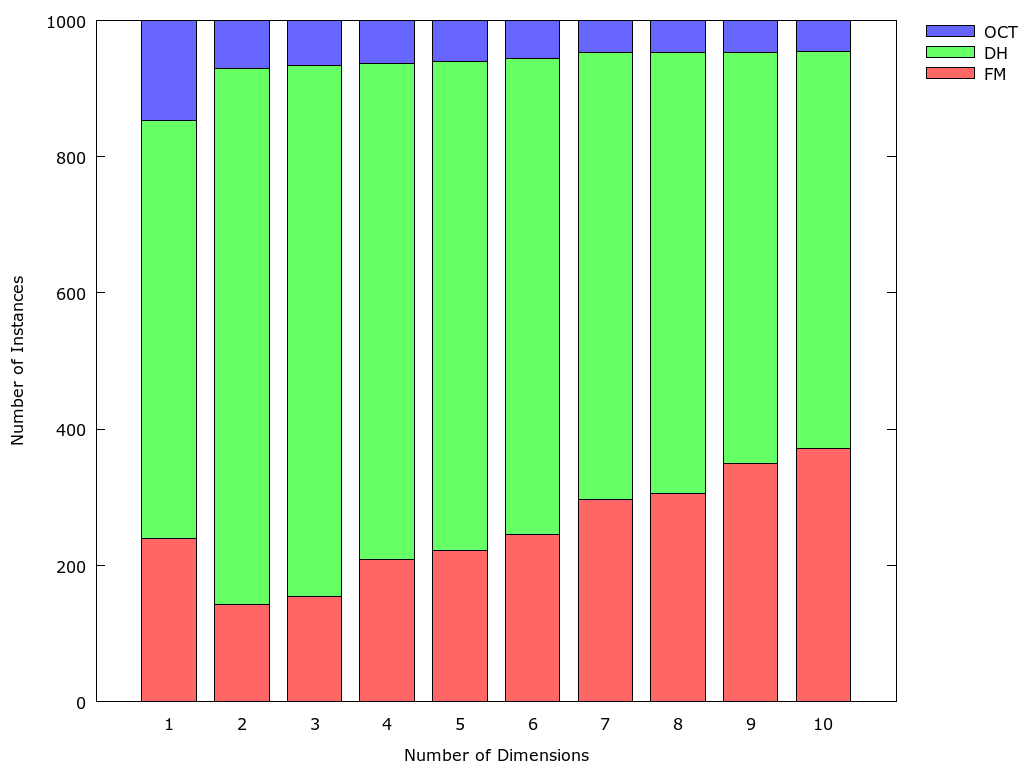}
    \caption{}
  \end{subfigure}
  ~
  \begin{subfigure}[b]{0.35\textwidth}
    \centering
	\includegraphics[width=1\textwidth]{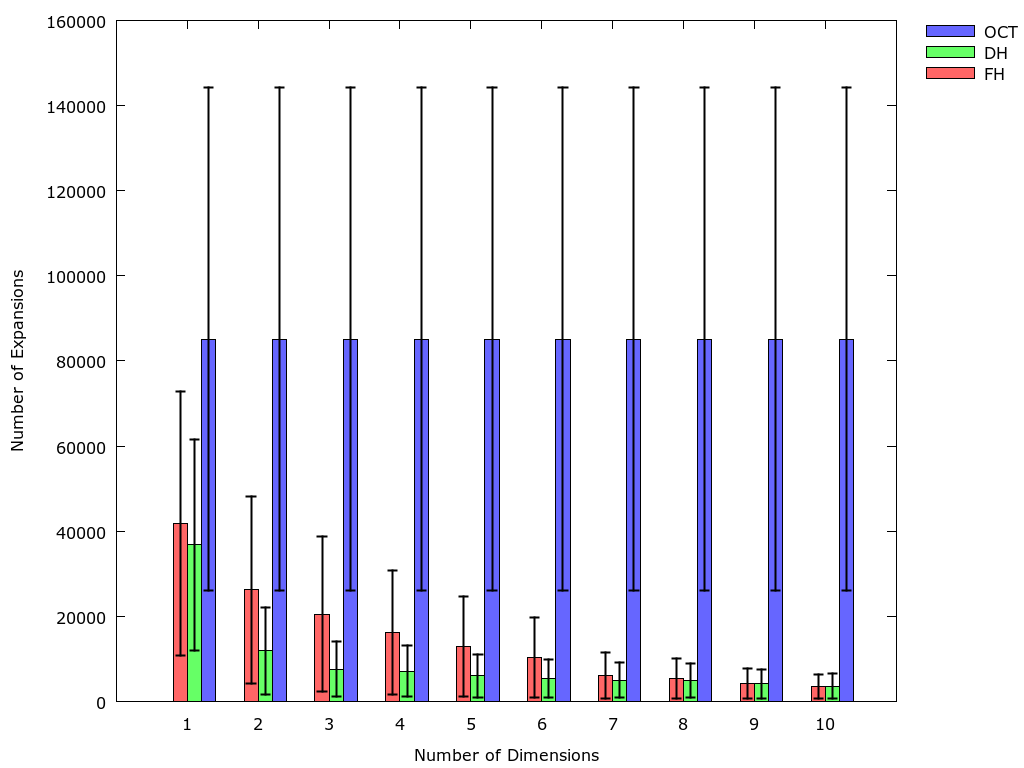}
    \caption{}
  \end{subfigure}
    
  \caption{Shows empirical results on $3$ maps from Bioware's Dragon Age: Origins. (a) is map `lak503d' containing $17,953$ nodes and $33,781$ edges; (d) is map `brc300d' containing $5,214$ nodes and $9,687$ edges; and (g) is map `maze512-32-0' containing $253,840$ nodes and $499,377$ edges. In (b), the x-axis shows the number of dimensions for the FastMap heuristic (or the number of pivots for the Differential heuristic). The y-axis shows the number of instances (out of $1,000$) on which each technique expanded the least number of nodes. Each instance has randomly chosen start and goal nodes. (c) shows the median number of expanded nodes across all instances. Vertical error bars indicate the MADs. The figures in the second and third rows follow the same order. In the legends, ``FM'' denotes the FastMap heuristic, ``DH'' denotes the Differential heuristic and ``OCT'' denotes the Octile heuristic.}
  \label{fig_experiments}
\end{figure*}

\section{Conclusions}
In this paper, we presented a near-linear time preprocessing algorithm, called FastMap, for producing a Euclidean embedding of a general edge-weighted undirected graph. At runtime, the Euclidean distances were used as heuristic by A* for shortest path computations. We proved that the FastMap heuristic is admissible and consistent, thereby generating shortest paths. FastMap produces the Euclidean embedding in near-linear time, which is significantly faster than competing approaches for producing Euclidean embeddings with optimality guarantees that run in cubic time. We also showed that it is competitive with other state-of-the-art heuristics derived in near-linear preprocessing time. However, FastMap has the combined benefits of requiring only near-linear preprocessing time and producing explicit Euclidean embeddings that try to recover the underlying manifolds of the given graphs.

\section{Acknowledgments}
The research at USC was supported by NSF under grant numbers 1724392, 1409987, and 1319966. The views and conclusions contained in this document are those of the authors and should not be interpreted as representing the official policies, either expressed or implied, of the sponsoring organizations, agencies or the U.S.  government.

\bibliographystyle{named}
\bibliography{LirBib}

\end{document}